\documentclass{article}
\usepackage{times}
\usepackage{helvet}
\usepackage{courier}
\usepackage{url} 
\usepackage{graphicx}
\usepackage{caption}
\usepackage{bm}
\usepackage{amsmath}
\usepackage{amssymb}
\usepackage{amsthm}
\usepackage{comment}
\usepackage{url}
\usepackage{algorithm}
\usepackage{algorithmic}
\usepackage{authblk}
\usepackage[round]{natbib}

\newtheorem{dfn}{Definition}
\newtheorem{thm}{Theorem}
\newtheorem{asm}{Assumption}

\newtheorem{lmm}{Lemma}

\author[1]{Masahiro Kato}
\author[1,2]{Liyuan Xu}
\author[2]{Gang Niu}
\author[2,1]{Masashi Sugiyama}

\affil[1]{The University of Tokyo}
\affil[2]{RIKEN}

\date{}

\usepackage[left=40truemm,right=40truemm]{geometry}

\begin{document}
%
\title{Alternate Estimation of a Classifier and the Class-Prior\\
from Positive and Unlabeled Data}
\maketitle
\begin{abstract}
We consider a problem of learning a binary classifier only from \emph{positive} data and \emph{unlabeled} data (\emph{PU learning}) and estimating the class-prior in \emph{unlabeled} data under the \emph{case-control} scenario.
Most of the recent methods of PU learning require an estimate of the \emph{class-prior probability} in unlabeled data,
and it is estimated in advance with another method.
However, such a two-step approach which first estimates the class-prior and then trains a classifier
may not be the optimal approach since the estimation error of the class-prior is not taken into account
when a classifier is trained.
In this paper, we propose a novel unified approach to estimating the class-prior and training a classifier \emph{alternately}.
Our proposed method is simple to implement and computationally efficient.
Through experiments, we demonstrate the practical usefulness of the proposed method.
\end{abstract} 

\section{Introduction}
We consider the problem of learning a binary classifier only from positive data and unlabeled data (PU learning). This problem arises in various practical situations, such as information retrieval and outlier detection \citep{elkan2008learning,ward2009presence,pmlr-v5-scott09a,blanchard2010semi,li2009positive,nguyen2011positive}. One of the theoretical milestones of PU learning is \citet{elkan2008learning} and there are subsequent researches called \emph{ unbiased PU learning} \citep{IEICE:duPlessis+Sugiyama:2014,ICML:duPlessis+etal:2015}, where the classification risk is estimated in an unbiased manner only from PU data.

We consider the \emph{case-control} scenario \citep{ward2009presence,elkan2008learning}, where positive data are obtained separately from unlabeled data and unlabeled data is sampled from the whole population. Under this setting, the true class-prior $\pi=p(y=+1)$ in unlabeled data is needed for the formulation of unbiased PU learning. However, since $\pi$ is often unknown in practice, methods for \emph{class-prior estimation in PU learning} have been developed. 

There are several existing methods for class-prior estimation in PU learning. \citet{elkan2008learning} is also a milestone of  class-prior estimation. Recently, \citet{christoffel2016class} proposed another method based on properly penalized Pearson divergences for class-prior estimation. \citet{pmlr-v48-ramaswamy16} gave a method for class-prior estimation based on kernel mean embedding. Furthermore, \citet{jain2016nonparametric} developed a method based on maximum likelihood estimation of the mixture proportion. At present, the method proposed by \citet{pmlr-v48-ramaswamy16} is reported as the best method in the performance for class-prior estimation.

While many methods have been proposed in both fields of PU learning and class-prior estimation, they considered estimating the class-prior and learning a classifier separately. However, such a two-step approach of first estimating the class-prior and then training a classifier may not be the best approach---an estimation error of the class-prior in the first step is not taken into account when a classifier is trained later. In this paper, we propose a method of \emph{alternately} estimating the class-prior and a classifier. The proposed algorithm is simple to implement, and it requires low computational costs. Experiments show that our proposed method is promising compared with existing approaches.

\section{Formulation of PU Learning}
In this section, we formulate the problem of PU learning and describe an unbiased PU learning method.

Suppose that we have a positive dataset and an unlabeled dataset i.i.d. as 
\begin{align*}
\label{data}
&\{\bm{x}_i\}_{i=1}^{n}\sim p(\bm{x}|y=+1),\nonumber\\
&\{\bm{x}'_i\}_{i=1}^{n'}\sim p(\bm{x})=\pi p(\bm{x}|y=+1)+(1-\pi)p(\bm{x}|y=-1),
\end{align*}
where $p(\bm{x}|y)$ is the class-conditional density, $p(\bm{x})$ is the marginal density, and $\pi=p(y=+1)$ is the unknown class-prior for the positive class. We assume that $\bm{x}_i$ and $\bm{x}'_i$ belong to a compact input space $\mathcal{X}$.
Let us define a classifier $h:\mathcal{X}\to \{-1,1\}$ as $h(\bm{x}) = {\tt sign} (f(\bm{x})-0.5)$, where $f(\bm{x})$ is a score function, i.e., $f:\mathcal{X}\to (0, 1)$, and ${\tt sign}(y)$ denotes the sign of $y$. Our goal is to obtain a classifier $h$ only from a positive dataset and an unlabeled dataset. 

The optimal classifier $f$ is given by minimizing the following functional on $f$ called the classification risk:
\[
R(f)=\mathbb{E}[\ell(f(X),Y)],
\]
where $\ell: \mathbb{R}\times \{\pm1\}\to \mathbb{R}$ is a loss function and $\mathbb{E}$ denotes the expectation over the unknown joint density $p(\bm{x},y)$.

\citet{ICML:duPlessis+etal:2015} showed that the risk can be expressed only with the positive and unlabeled data as
\begin{align*}
&R(f)=\pi \mathbb{E}_1[\ell(f(X), +1)] - \pi \mathbb{E}_1[\ell(f(X), -1)] +\mathbb{E}_X[\ell(f(X), -1)],
\end{align*}
where $\mathbb{E}_1$ and $\mathbb{E}_X$ are the expectations over $p(\bm{x}|y=+1)$ and $p(\bm{x})$ respectively. From this expression, we can easily obtain an unbiased estimator of the classification risk from empirical data, by simply replacing the expectations by the corresponding sample averages.

We use the logarithmic loss function, i.e., $\ell(f(\bm{x}), +1))=-\log(f(\bm{x}))$ and $\ell(f(\bm{x}), -1)=-\log(1-f(\bm{x}))$ for $\ell$. Then the above risk can be expressed as
\begin{align}
\label{basic}
&R(f) = -\pi\mathbb{E}_{1} [\log(f(X))] + \pi\mathbb{E}_1[\log(1-f(X))]-\mathbb{E}_{X}[\log(1-f(X))].
\end{align}
We can derive some other loss functions from the logarithmic loss function and $f(\bm{x})$. For example, if we use the sigmoid function as the score function, i.e., $f(\bm{x})=\frac{1}{1+\exp(-g(\bm{x}))}$, where $g$ is a function such that $g:\mathcal{X}\rightarrow \mathbb{R}$. In this case, the loss becomes the logistic loss: $\ell(f(\bm{x}), +1))=\log(1 + \exp(-g(\bm{x})))$ and $\ell(f(\bm{x}), -1))=\log(1 + \exp(g(\bm{x})))$. 

\section{Alternate Estimation}
In this section, we propose our algorithm for learning a classifier and estimating the class-prior alternately only from PU data.

Let us regard the class-prior as a parameter and denote it by $\kappa$. We define the risk $R^{\kappa}(f)$ as
\begin{align*}
&R^{\kappa}(f)= -\kappa\mathbb{E}_{1} [\log(f(X))] + \kappa\mathbb{E}_1[\log(1-f(X))]-\mathbb{E}_{X}[\log(1-f(X))],
\end{align*}
and denote its empirical version by $\hat{R}^{\kappa}(f)$.

\subsection{Algorithm in Population}
In this subsection, we assume an access to infinite samples and propose an algorithm in population. After the theoretical analysis of the algorithm in population, we also define the empirical version of the algorithm. The algorithm trains a classifier $f$ and estimates the class-prior $\pi$ alternately by iterating the following steps:
\begin{itemize}
\item Given estimated class-prior $\pi^{*}$, train a classifier $h$ by finding $f$ that minimizes the empirical risk $R^{\pi^{*}}(f)$.
\item Let us denote the minimizer as $f^{*}$. Treat $f^{*}$ as an approximation of $p(y=+1|\bm{x})$ and update $\pi^{*}$ by taking the expectation of $f^{*}$ over unlabeled data.
\end{itemize}
We denote the estimated class-prior after the $k$th iteration as $\pi^{(k)}$. In the initial step, we set an initial prior $\pi^{(0)}$ such that $\pi^{(0)} > \pi$. Then, we iterate the above alternate estimation until convergence. Intuitively, this is based on the following relationship for the true conditional probability $p(y=+1|\bm{x})$.
\begin{align*}
\pi = \int p(y=+1|\bm{x}) p(\bm{x}) \mathrm{d}\bm{x}.
\end{align*}

In the rest of this section, we theoretically justify this estimation procedure.

\subsection{Convergence to Non-negative Class-prior}
In this subsection, we investigate theoretical properties of our algorithm in population for the true risk minimizer. 

\paragraph{Assumptions:} Firstly we put the following assumptions on the classification model $f$, $p(y=+1|\bm{x})$, $p(\bm{x})$ and $p(\bm{x}|y=+1)$. Let us denote the function space of $f$ by $\mathcal{F}$.
\begin{asm}
\label{asm1}
$f$ is a continuously differentiable function $f:\mathbb{R}^{d}\rightarrow (0, 1-\epsilon]$. 
\end{asm}
\begin{asm}
\label{asm2}
Probability density function $p(\bm{x})$ is $L_1$-Lipchitz continuous, and $p(\bm{x}|y=+1)$ is $L_2$-Lipchitz continuous. This means that the following inequalities hold for all $x,x' \in \mathbb{R}^d$:
\begin{align*}
    &|p(\bm{x}) - p(\bm{x}')| \leq L_1 \|\bm{x} - \bm{x}'\|_1,\\
    &|p(\bm{x}|y=+1) - p(\bm{x}'|y=+1)| \leq L_2 \|\bm{x} - \bm{x}'\|_1,
\end{align*}
where $\|\cdot\|_1$ denotes the $l_1$-norm.
\end{asm}
We use Assumption 1 to derive a stationary point based on the Euler-Lagrange equation. Assumption 2 guarantees convergence to the true class-prior by assuming the smoothness of the probability distributions.

Then, we introduce the following quantity $\pi_{\text{max}}$, which represents the largest possible class-prior given the positive and unlabeled data.
\begin{dfn}[Non-negative class-prior]
Let us define $\pi_{\mathrm{max}}$ as follows:
\begin{align*}
    \pi_{\mathrm{max}} = \sup \{\alpha | \forall \bm{x},~ (1-\epsilon)p(\bm{x}) \geq \alpha p(\bm{x}|y=+1)\}.
\end{align*}
\end{dfn}
As \citet{ward2009presence} proved, if we do not put any assumptions, class-prior estimation
in the case-control scenario is an intractable task. Hence, we also put the following assumption on the true class-prior.
\begin{asm}
\label{asm3}
\begin{align*}
\pi = \pi_{\max}
\end{align*}
\end{asm}

Assumption 3 makes us possible to identify the class-prior. The idea of estimating $\pi_{\mathrm{max}}$ shared with other existing methods such as \citet{christoffel2016class}.

Next, we prove that our algorithm defined with population converges to $\pi_{\mathrm{max}}$ when $\hat{\pi}^{(0)}$ is large enough. 

\subsubsection{Theoretical Convergence Guarantee:}
Here, we prove our main theorem which guarantees the convergence of our class-prior estimator to the true class-prior. Given an estimator of the class-prior $\pi^{(k)}$ at the $k$th step, the optimization problem for learning a classifier is written as follows:
\begin{align}
\label{OPT}
&\min\ R^{\pi^{(k)}}(f)\nonumber\\
&\mathrm{s.t.} \ \ \ 0 \leq f \leq 1-\epsilon.
\end{align}
Let us denote the solution of problem (\ref{OPT}) as $f^{(k+1)*}$. Then, we obtain a new estimator of class-prior $\pi^{(k+1)}$ as follows:
\begin{align*}
\pi^{(k+1)} = \int f^{(k+1)*}p(\bm{x})\mathrm{d}\bm{x}.
\end{align*}

In order to prove the theorem, we prove the following two lemmas. The proofs of the lemmas are shown in Appendix.
\begin{lmm}\label{lemma1} Suppose that Assumption 1 holds. The stationary point of the risk minimization problem is given as follows:
\begin{align*}
        f^{(k+1)*}(\bm{x}) =  \begin{cases}
          \frac{\pi^{(k)}p(\bm{x}|y=+1)}{p(\bm{x})} &(\bm{x}\in D^{(k)}),\\
           1-\epsilon &(\bm{x}\notin D^{(k)}),
        \end{cases}
    \end{align*}
  where $D^{(k)} = \{x|\pi^{(k)}p(\bm{x}|y=1) \leq (1-\epsilon)p(\bm{x}) \}$.
\end{lmm}

\begin{lmm}\label{lemma2}
If $\pi^{(k)} > \pi_{\mathrm{max}}$, then we have 
\begin{align*}
    & \pi^{(k)} - \pi^{(k+1)} = \int \left({\pi^{(k)}p(\bm{x}|y=+1)} - (1-\epsilon)p(\bm{x})\right)_+ \mathrm{d}\bm{x},
\end{align*}
where $(a)_+ = \max(a,0)$.
\end{lmm}
Lemma \ref{lemma1} implies that the function given as a stationary point approaches to $\frac{\pi^{(k)}p(\bm{x}|y=+1)}{p(\bm{x})}$, but it is truncated by $1$ if $\frac{\pi^{(k)}p(\bm{x}|y=+1)}{p(\bm{x})} \geq 1$. This lemma explicates the behavior of a classifier trained under an incorrectly estimated class-prior. Lemma \ref{lemma2} shows the difference between an estimated class-prior used
for training a classifier and another class-prior estimated by the trained classifier. 

We prove the following theorem, which shows that, in population, the convergence of our class-prior estimator to the true class-prior. The theorem follows directly from the lemmas and the assumptions mentioned above. This proof is also shown in Appendix.
\begin{thm}
For the initial class-prior $\pi^{(0)}$ such that $\pi^{(0)} > \pi$, $\pi^{(k)}$ in Algorithm 1 converges to  $\pi_{\mathrm{max}}$ as $k\to\infty$.
\end{thm}

\paragraph{Necessity of $\epsilon$:} Readers might feel that we do not have to consider $\epsilon>0$; rather we might simply consider $f:R^{d}\rightarrow[0,1]$. However, it causes a problem because, in that case, there is no stationary point in our optimization problem due to the existence of $\log(1-f(\bm{x}))$ in (\ref{basic}).

\subsection{Algorithm}
Here, we define an empirical version of our algorithm. Let us denote the the solution of the empirical risk minimization problem at the $k$th step as $\hat{f}^{(k+1)}$ and the estimated class-prior after the $k$th iteration as
$\hat{\pi}^{(k+1)}$. In the initial step, we set an initial prior $\hat{\pi}^{(0)}$ such that
$\hat{\pi}^{(0)} > \pi$. Then, we iterate the following alternate estimation procedure. 
\begin{itemize}
\item At the $k$th step, given estimated class-prior $\hat{\pi}^{(k)}$, train a classifier $h$ by finding $f$ that minimizes the empirical risk $\hat{R}^{\hat{\pi}^{(k)}}(f)$, where $\hat{R}^{\hat{\pi}^{(k)}}(f)$ is an empirical estimator of (\ref{basic}) with $\pi$ replaced by $\hat{\pi}^{(k)}$.
\item Treat $\hat{f}^{(k+1)}$ as an approximation of $p(y=+1|\bm{x})$ and obtain $\hat{\pi}^{(k+1)}$ by the average of $\hat{f}^{(k+1)}$ over samples from unlabeled data.
\end{itemize}

However, unlike the property in population, choice of the initial class-prior has influence on the convergence. We observed that, if the initial class-prior $\hat{\pi}^{(0)}$ is much larger than the true class-prior $\pi$, $\hat{f}^{(k)}$ tends to approach to $1$ as $k\rightarrow +\infty$. For example, we observed this phenomenon when the initial class-prior was $\hat{\pi}^{(0)} = 0.9$, but the true class-prior was $\pi=0.2$. This phenomenon is considered to be related with the estimation error and violation of the assumptions. In order to avoid this phenomenon, we introduce $\delta$ and $\xi$ as heuristics. $\delta$ should be a large value which is less than $1$, but close to $1$. $\xi$ should be a small positive value. If $\hat{f}^{(k)}(\bm{x}) = 1-\epsilon$ for all $\bm{x}$, the initial class-prior is wrong. Hence, we reset $\hat{\pi}^{(0)}$ by $\hat{\pi}^{(0)}-\xi$ when $\hat{\pi}^{(k)} > \delta$. Then, we restart our algorithm under a new estimator of the class-prior $\hat{\pi}^{(0)}$. A pseudo-code of our algorithm is shown in Algorithm 1. 
\begin{algorithm}[tb]
   \caption{Alternate Estimation}
   \label{alg}
\begin{algorithmic}
   \STATE {\bfseries Input:} Set an initial class-prior $\hat{\pi}^{(0)} > \pi$.
   \STATE Set a value $\delta < 1$ and a small positive value $\xi$. 
   \STATE $k = 0$.
   \REPEAT
   \STATE Estimate $\hat{f}^{(k+1)} = \mathrm{arg}\min_{f\in\mathcal{F}}\hat{R}^{\hat{\pi}^{(k)}}(f)$.
   \STATE $\hat{\pi}^{(k+1)} = \frac{1}{n'}\sum^{n'}_{i=1}\hat{f}^{(k+1)}(\bm{x}_i)$.
   \STATE If $\hat{\pi}^{(k+1)} > \delta$, $\hat{\pi}^{(0)} = \hat{\pi}^{(0)} - \xi$ and $\hat{\pi}^{(k+1)} = \hat{\pi}^{(0)}$.
   \STATE $k = k + 1$
   \UNTIL $\hat{\pi}^{(k)}$ converges.
\end{algorithmic}
\end{algorithm} 

Our algorithm is similar to expectation-maximization (EM) algorithm \citep{Dempster77maximumlikelihood}. However, there is a crucial difference between our algorithm and EM algorithm. EM algorithm aims to maximize the likelihood, but our algorithm is based on the property of the stationary point to estimate the class-prior and do not consider the maximization of the likelihood through iterations. For this reason, although our algorithm is similar to EM algorithm, the mechanism of the estimation is quite different.

\section{Experiments}
In this section, we report experimental results which were conducted using numerical and real datasets\footnote{All codes of experiments can be downloaded from \url{https://github.com/MasaKat0/AlterEstPU}.}. In Comparison Tests and Benchmark Tests, we used $6$ classification datasets from UCI repository\footnote{We use the {\tt digits} (or called {\tt mnist}), {\tt mushrooms}, {\tt spambase}, {\tt waveform}, {\tt usps}, {\tt connect-4}. Some data are multi-labeled data, so we divide them into two groups. The UCI data were downloaded from \url{https://archive.ics.uci.edu/ml/index.php} and \url{https://www.csie.ntu.edu.tw/~cjlin/libsvmtools/}.}. The details of datasets are given in Table 1. In all experiments, we set $\delta=0.9$ and $\xi=0.01$. In Comparison Tests and Benchmark Tests, we iterated $150$ times in alternate estimation to make sure that it converges. In practice, it is not necessary to iterate as many steps as we did in the experiments. 

We used the sigmoid function for representing the probabilistic model $f$:
\begin{align*}
f(\bm{x}) = \frac{1}{1+\exp(-g(\bm{x}))},
\end{align*}
where $g(\bm{x})$ is a real-valued function. 

For $g(\bm{x})$, we used two linear models. We denote the parameter of models as $\bm{\theta}$ and the parameter space as $\Theta$. The first model is
\begin{align}
\label{linearmodel1}
g(\bm{x}) = \bm{\theta}^\top\bm{z},
\end{align}
where $\bm{z} = (1,\bm{x}^{\top})^{\top}$ and $^\top$ denotes the transpose. This model has $(d+1)$-dimensional parameters $\bm{\theta}$ when $\bm{x}$ is a $d$-dimensional vector. The second model is
\begin{align}
\label{linearmodel2}
g(\bm{x}) = \bm{\theta}^\top\bm{x}.
\end{align}
This model does not use the bias term. The second model might be unnatural, but empirically it performed well as we show below. We refer to the first model as AltEst1 and the second model as AltEst2 respectively. In the Numerical Tests and the Comparison Tests section, we used only AltEst1. In the Benchmark Tests section, we used both models. We assume the parameter space $\Theta$ is a compact set. Because we also assume the input space $\mathcal{X}$ is a compact set, the models (\ref{linearmodel1}) and (\ref{linearmodel2}) satisfy Assumption 1.

\begin{table}
\label{Dataset}
\caption{Specification of datasets}
\begin{center}
\scalebox{1.2}[1.2]{\begin{tabular}{cccc}
\hline
Dataset&\# of samples&Class-prior&Dimension\\
\hline
{\tt waveform} & 5000 & 0.492 &  21 \\
{\tt mushroom} & 8124 & 0.517 &  112 \\
{\tt spambase} & 4601 & 0.394 &  57 \\
{\tt digits (mnist)}& 70000  & 0.511 &  784 \\
{\tt usps}      & 9298 & 0.524 &  256 \\
{\tt connect-4}& 67557& 0.658 &  126 \\
\end{tabular}}
\end{center}
\end{table}
\subsection{Numerical Tests}
In this subsection we numerically illustrate the convergence to the true class-prior of AltEst1. We used samples from a mixture distribution of the following two class-conditional distributions:
\begin{align*}
&p(x|y=+1)=\mathcal{N}_{x}(2,1^2)\ \nonumber\\
&\mathrm{and}\ p(x|y=-1)=\mathcal{N}_{x}(-2,1^2),
\end{align*}
where $\mathcal{N}(\mu, \sigma^2)$ denotes the univariate normal distribution with mean $\mu$ and variance $\sigma^2$.

\paragraph{Behavior of Classifiers: }
This experiment shows how to approach the true classifier through alternate estimation. We generated $100$ positive samples and $10000$ unlabeled samples. We made three datasets with different class-priors $\pi = 0.2$, $\pi = 0.5$, and $\pi = 0.8$. We plotted the classifiers estimated from the initial class-prior $0.9$ in each round (in total $10$ rounds) in Fig. $1$. 

\paragraph{Behavior of Updated Class-Priors: }
This experiment shows how to move the estimated class-prior after the optimization under an inaccurate class-prior. We generated $100$ positive samples and generated $10000$ unlabeled samples with the different class-prior $\pi = 0.2$, $\pi = 0.5$, and $\pi = 0.8$. In Fig $2$, the horizontal axis represents the inserted class-prior and the vertical axis represents the updated class-prior after one iteration using the inserted class-prior. The blue line represents $y=x$ and visualizes the fixed points. If the estimated class-prior is on the line in a round, the estimated class-prior will not change after one iteration.

In our two Gaussian datasets, we believe that the non-negative class-prior matched the true class-prior and the theoretical result explained the behavior of our algorithm.

\begin{figure*}[htb]
\begin{center}
\includegraphics[width = 14.cm]{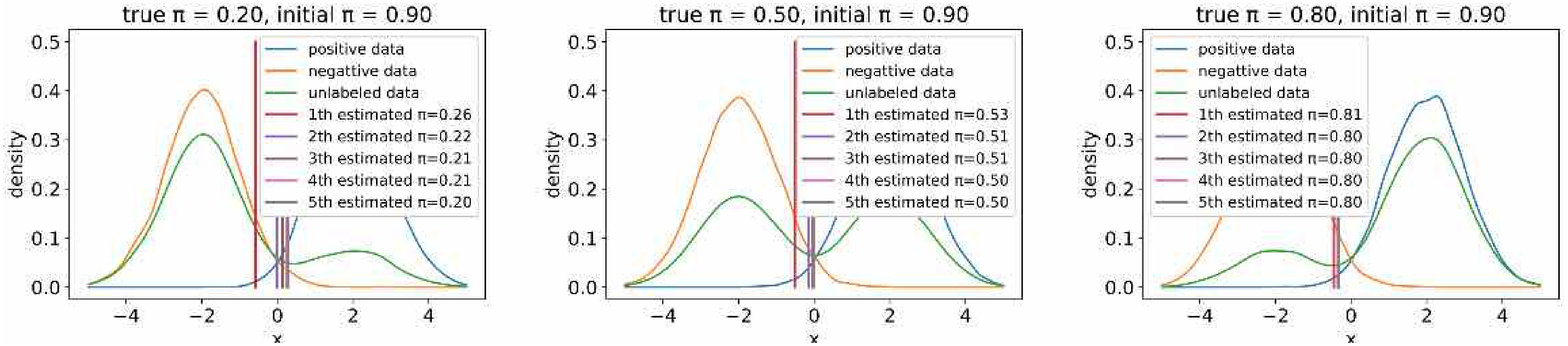}

\caption{Behavior of classifiers. The horizontal axis is value of $x$ and the vertical axis is probabilistic density. The vertical line represents the classifier.}
\end{center}
\bigskip
\begin{center}
\includegraphics[width = 14.cm]{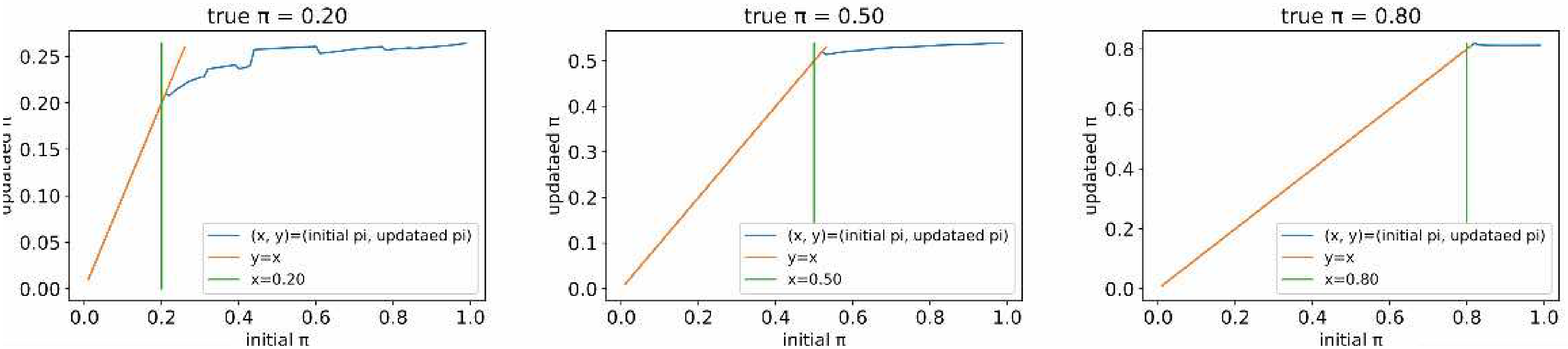}
\caption{Behavior of updated class-priors. The horizontal axis is the initial class-prior and the vertical axis is the returned class-prior after one iteration.}
\end{center}
\end{figure*}

\subsection{Comparison Tests}
To demonstrate the usefulnes of the proposed algorithm compared with two methods proposed by \citet{pmlr-v48-ramaswamy16}, "KM1" and "KM2", which are based on kernel mean embedding. We used the {\tt digits}, {\tt usps}, {\tt connect-4}, {\tt mushroom}, {\tt waveform}, and {\tt spambase} datasets from UCI repository. For the {\tt digits}, {\tt usps}, {\tt connect}, and {\tt mushroom} datasets, we projected data points onto the $50$-dimensional space by principal component analysis (PCA). For each binary labeled dataset, we made $8$ different pairs of positive and unlabeled data. Firstly, given the binary labeled dataset, we made $4$ pairs of positive and unlabeled data with the different class-priors for the unlabeled dataset. The class-prior was chosen from $\{0.2, 0.4, 0.6, 0.8\}$. Next, we flipped the labels, i.e., we used negative data as positive data and positive data as negative data, and made another $4$ datasets with the different class-priors. As a result, we obtained $8$ pairs of positive and unlabeled data. We evaluated the performance by the absolute error defined as $|\hat{\pi} - \pi|$, where $\hat{\pi}$ is an estimated class-prior. For each pair of $8$ pairs made from $6$ datasets, we estimated the class-prior $5$ times and calculated the average absolute error. The number of positive data was fixed at $400$. The numbers of unlabeled data were $400$, $800$, $1600$, and $3200$. This setting is almost all same as \citet{pmlr-v48-ramaswamy16}. The result is shown in Fig $3$. The horizontal axis is corresponding to the number of unlabeled data and the vertical axis is corresponding to the average absolute error $|\hat{\pi} - \pi|$. We set the initial class-prior of our algorithm as $0.9$. 

Our algorithm has preferable performance compared with the existing method in several cases. The reason why our algorithm could not work well in some datasets such as the  {\tt waveform} might be due to the violation of assumptions.

\begin{figure}[htb]
\begin{center}
\includegraphics[width = 14.0cm]{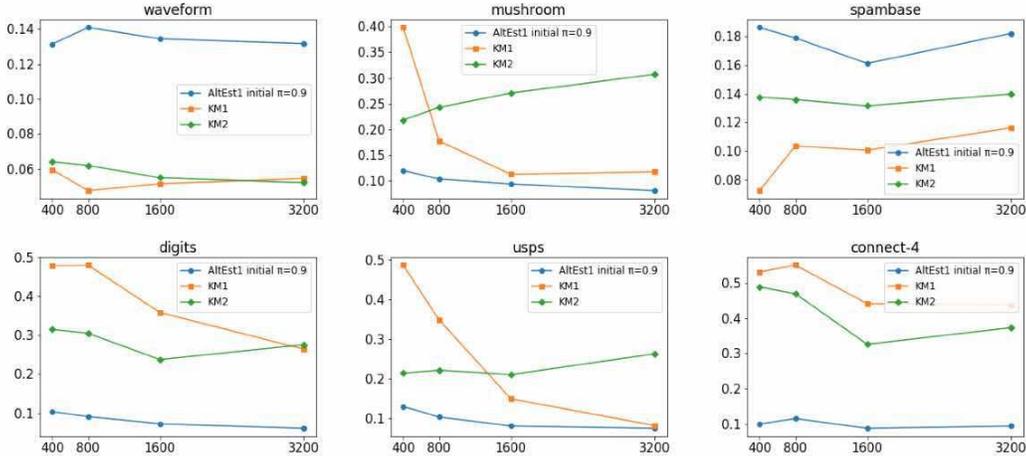}
\caption{Comparison: The horizontal axis is the number of unlabeled data and the vertical axis is the averaged absolute error $|\hat{\pi} - \pi|$. We set the initial class-prior of our algorithm as $0.9$.}
\end{center}
\end{figure}

\subsection{Benchmark Tests}
In this subsection, we investigate the experimental performance in more detail. Unlike the previous experiment, we report estimators of the class-prior instead of the average absolute error rate. We used the {\tt digits}, {\tt mushrooms}, {\tt usps}, {\tt waveform} and {\tt spambase} datasets. For {\tt digits}, {\tt mushrooms} and {\tt usps} datasets, we reduced the dimension by PCA. For the {\tt digits} dataset, we projected the data points onto the $100$-dimensional space and the $200$-dimensional space. For the {\tt mushrooms} and {\tt usps} datasets, we projected the data points onto the $100$-dimensional data space. For each dataset, we drew $400$ positive data and $1600$ unlabeled data. After learning a classifier and estimating the class-prior, we also checked the accuracy of the classifier using test data. For the {\tt digits}, {\tt mushrooms} and {\tt usps} datasets, we used $1000$ test data. For the {\tt waveform} and {\tt spambase} datasets, we used $300$ test data since the size of the original datasets were limited. The result is shown in Tables $2 - 4$ with the estimated class-prior and the error rate of classifiers. To evaluate the accuracy of a classifier, we compared classifiers trained by AltEst1 and AltEst2 with three classifiers trained by convex PU learning with logistic loss and the model (\ref{linearmodel1}) given the true class-prior and two estimated class-priors by KM1 and KM2. To evaluate the performance of class-prior estimation, we compared our algorithm with the methods proposed by \citet{pmlr-v48-ramaswamy16}. We set the initial class-prior of our algorithm as $0.9$. We ran the experiments $100$ times and calculated the mean and standard deviation. We evaluated the performance of classifiers and, by reducing by PCA, we also checked the robustness of the algorithms to the dimension. 

As shown in Tables $2 - 4$, AltEst2 returned a preferable estimator of the class-prior. We can observe that our algorithms work stably in various settings, but the existing methods were easily influenced by the true class prior and the dimensionality of data. For the {\tt waveform} and {\tt spambase} datasets, our algorithms could not show good results. For these datasets, the accuracies of the classifier were low even given the true class-prior. We believe that the accuracy of the classifier affects the estimation of the class-prior, i.e., it is difficult to estimate the class-prior with data which is difficult to be classified. In practice, we need to specify the model more carefully to gain accuracy of a classifier. 

\section{Discussion and Conclusion}
In this paper, we proposed a novel unified approach to estimating the class-prior and training a classifier alternately. Our method has benefits from the one-step estimation compared with other conventional two-step approaches which estimate the class-prior firstly and then train a classifier. We showed the theoretical guarantees in population and proposed an algorithm. In experiments, our method showed preferable performances. Moreover, we confirmed that, if we set the initial class-prior as the value that is close to the true class-prior, the behavior of our algorithm are improved in practice. An important future direction is to extend our method to adopt deep neural networks to gain more accuracy.

\section*{Acknowledgments}
MS was supported by JST CREST JPMJCR1403.

\bibliography{arxiv1.bbl}
\bibliographystyle{plainnat}

\begin{figure}[htb]
\begin{center}
\scalebox{0.7}[0.7]{
  \begin{tabular}{lc||cccc||cccc}
    \hline
          &  & {\tt 0 vs. 1} & & & &{\tt 0 vs. 2} & & &  \\
    \hline
    True  & Prior & 20 (-) & 40 (-) & 60 (-) & 80 (-) & 20 (-) & 40 (-) & 60 (-) & 80 (-) \\
          & Err & 0.2 (.006) & 0.6 (.003) & 1.3 (.006) & 2.4 (.014) & {\bf 2.4} (.013) & 5.5 (.016) & 8.9 (.019) & 12.2 (.025) \\
    AltEst1   & Prior & 19.3 (.017) & 38.4 (.029) & {\bf 60.2} (.077) & 64.9 (.099) & 14.9 (.035) & 36.1 (.044) & 42.2 (.028) & 56.7 (.040) \\
          & Err & 0.7 (.018) & 1.4 (.028) & 6.5 (.052) & 15.2 (.088) & 6.3 (.032) & 6.3 (.046) & 21.6 (.031) & 26.0 (.042) \\
    AltEst2 & Prior & {\bf 19.8} (.002) & {\bf 39.8} (.002) & {\bf 59.9} (.002) & {\bf 80.0} (.003) & 15.9 (.017) & 28.4 (.048) &   {\bf 57.5} (.034) & {\bf 81.2} (.016) \\
          & Err & {\bf 0.1} (.000) & {\bf 0.1} (.001) & {\bf 0.2} (.001) & {\bf 0.6} (.002) & 4.8 (.017) & 12.7 (.047) & {\bf 5.5} (.035) & {\bf 7.1} (.012) \\
    KM1   & Prior & 9.6 (.012) & 38.9 (.029) & {\bf 59.6} (.015) & 81.1 (.022) & 23.5 (.009) & 45.4 (.013) & 67.4 (.019) & 29.3 (.022) \\
          & Err & 8.4 (.017) & 1.3 (.030) & 1.1 (.009) & 5.6 (.038) & 3.0 (.008) & 6.4 (.013) & 19.6 (.031) & 58.3 (.049) \\
    KM2   & Prior & 16.1 (.012) & 37.3 (.021) & 53.9 (.020) & 63.4 (.034) & {\bf 22.3} (.007) & {\bf 39.8} (.012) & 52.9 (.023) & 3.0 (.055) \\
          & Err & 4.0 (.013) & 2.9 (.032) & 7.0 (.039) & 17.1 (.040) & 2.7 (.008) & {\bf 5.0} (.013) & 9.0 (.038) & 79.4 (.022) \\
    \hline
          &  & {\tt 0 vs. 3} & & & &{\tt 0 vs. 4} & & &  \\
    \hline
    True  & Prior & 20 (-) & 40 (-) & 60 (-) & 80 (-) & 20 (-) & 40 (-) & 60 (-) & 80 (-) \\
          & Err & {\bf 1.6} (.008) & 3.3 (.010) & 7.2 (.017) & 10.4 (.019) & {\bf 0.7} (.003) & 2.0 (.007) & 5.3 (.013) & 8.9 (.021) \\
    AltEst1  & Prior & 17.1 (.036) & 36.6 (.033) & 40.8 (.029) & 54.6 (.043) & {\bf 19.7} (.012) & 38.4 (.007) & 42.0 (.036) & 53.6 (.038) \\
          & Err & 3.6 (.029) & 3.7 (.033) & 20.6 (.033) & 26.8 (.050) & 0.8 (.011) & 1.6 (.006) & 18.7 (.035) & 25.4 (.039) \\
    AltEst2 & Prior & 16.0 (.023) & 34.2 (.048) & {\bf 59.4} (.008) & {\bf 82.9} (.011) & 17.7 (.009) & {\bf 39.0} (.006) &   {\bf 60.2} (.005) & {\bf 81.3} (.005) \\
          & Err & 4.0 (.019) & 5.6 (.048) & {\bf 3.0} (.007) & {\bf 6.0} (.010) & 1.6 (.006) & {\bf 0.9} (.004) & {\bf 1.4} (.004) & {\bf 3.0} (.0006) \\
    KM1   & Prior & 23.4 (.009) & 45.0 (.011) & 67.1 (.017) & 29.3 (.022) & 21.8 (.008) & 42.8 (.011) &  64.1 (.015) & 36.2 (.081) \\
          & Err & 1.9 (.006) & 3.6 (.011) & 17.2 (.021) & 58.0 (.055) & 0.8 (.003) & 1.8 (.007) & 10.4 (.018) & 44.6 (.085) \\
    KM2   & Prior & {\bf 22.2} (.009) & {\bf 40.1} (.011) & 53.5 (.020) & 7.3 (.088) & 21.0 (.008) & 38.8 (.012) & 53.0 (.018) & 14.7 (.109) \\
          & Err & 1.8 (.005) & {\bf 3.0} (.009) & 7.0 (.039) & 77.8 (.035) & 0.8 (.003) & 1.6 (.006) & 9.4 (.036) & 73.5 (.084) \\
    \hline
          &  & {\tt 0 vs. 5} & & & &{\tt 0 vs. 6} & & &  \\
    \hline
    True  & Prior & 20 (-) & 40 (-) & 60 (-) & 80 (-) & 20 (-) & 40 (-) & 60 (-) & 80 (-) \\
          & Err & 4.8 (.015) & {\bf 5.7} (.014) & 10.8 (.019) & 13.3 (.025) & 3.8 (.010) & 6.4 (.015) & 7.8 (.015) & 10.2 (.019) \\
    AltEst1   & Prior & {\bf 18.0} (.036) & 34.3 (.054) & 41.8 (.033) & 56.4 (.040) & 15.7 (.034) & 35.3 (.047) & 42.8 (.027) & {\bf 57.1} (.053) \\
          & Err & 4.8 (.022) & 9.1 (.045) & 20.2 (.030) & 25.2 (.035) & 6.5 (.023) & 8.7 (.040) & 20.3 (.030) & 25.0 (.049) \\
    AltEst2 & Prior & 10.9 (.027) & 26.1 (.073) & 47.4 (.119) & {\bf 80.7} (.022) & 16.9 (.022) & 35.9 (.029) &   {\bf 61.0} (.008) & {\bf 84.1} (.012) \\
          & Err & 9.0 (.020) & 14.9 (.065) & 14.7 (.110) & {\bf 11.2} (.017) & 4.6 (.018) & 6.6 (.027) & {\bf 6.1} (.010) & {\bf 7.5} (.008) \\
    KM1   & Prior & 25.1 (.010) & 47.9 (.016) & 69.6 (.018) & 33.0 (.065) & 23.4 (.010) & 44.7 (.014) &  66.0 (.015) & 50.2 (.064) \\
          & Err & 3.6 (.006) & 8.1 (.057) & 24.6 (.044) & 49.7 (.078) & 3.9 (.006) & 6.7 (.013) & 14.5 (.020) & 30.3 (.041) \\
    KM2   & Prior & 22.8 (.010) & {\bf 40.2} (.014) & {\bf 52.7} (.019) & 3.9 (.065) & {\bf 22.2} (.008) & {\bf 39.8} (.013) & 53.6 (.015) & 22.6 (.054) \\
          & Err & {\bf 3.2} (.005) & {\bf 5.7} (.012) & {\bf 10.0} (.031) & 79.0 (.028) & {\bf 3.7} (.005) & {\bf 6.2} (.013) & 8.8 (.028) & 69.5 (.071) \\
    \hline
          &  & {\tt 0 vs. 7} & & & & {\tt 0 vs. 8} & & &  \\
    \hline
    True  & Prior & 20 (-) & 40 (-) & 60 (-) & 80 (-) & 20 (-) & 40 (-) & 60 (-) & 80 (-) \\
          & Err & {\bf 1.9} (.008) & 4.0 (.014) & 5.2 (.013) & 7.8 (.020) & 3.3 (.012) & 5.2 (.010) & 10.7 (.019) & 13.5 (.017) \\
    AltEst1   & Prior & 16.4 (.036) & 36.1 (.040) & 43.3 (.073) & 57.2 (.055) & 17.6 (.039) & 34.9 (.052) & 41.0 (.028) & 53.3 (.040) \\
          & Err & 4.3 (.031) & 5.4 (.041) & 19.0 (.044) & 23.7 (.056) & 4.4 (.023) & 7.3 (.041) & 20.2 (.030) & 27.3 (.038) \\
    AltEst2 & Prior & 16.8 (.024) & 36.7 (.040) & {\bf 59.3} (.006) & {\bf 81.5} (.008) & 16.0 (.020) & 30.2 (.043) &   {\bf 57.1} (.019) & {\bf 80.9} (.012) \\
          & Err & 3.3 (.022) & {\bf 3.9} (.039) & {\bf 2.7} (.006) & {\bf 4.7} (.010) & 4.6 (.015) & 10.3 (.040) & {\bf 6.7} (.016) & {\bf 9.6} (.012) \\
    KM1   & Prior & 21.0 (.010) & {\bf 41.8} (.014) & 63.0 (.016) & 52.4 (.057) & 23.3 (.008) & 45.2 (.014) & 66.9 (.020) & 30.4 (.027) \\
          & Err & {\bf 1.9} (.005) & {\bf 4.2} (.012) & 9.1 (.020) & 27.4 (.039) & 3.2 (.004) & 5.5 (.011) & 21.0 (.029) & 54.4 (.051) \\
    KM2   & Prior & {\bf 20.2} (.008) & 37.9 (.014) & 52.6 (.018) & 23.2 (.035) & {\bf 22.1} (.008) & {\bf 39.8} (.013) & 53.5 (.019) & 7.5 (.091) \\
          & Err & {\bf 1.9} (.007) & {\bf 4.0} (.024) & 8.0 (.033) & 68.4 (.049) & {\bf 3.0} (.004) & {\bf 5.0} (.009) & 9.0 (.032) & 77.0 (.048) \\
  \end{tabular}
  }
  \captionsetup{labelformat=empty}
  \caption{Table 2: {\tt digits} 0 vs. 1, 2, 3, 4, 5, 6, 7, 8 with dimension $100$ via PCA: The estimated class-prior (Prior: $\%$) and the error rate of classification in test data (Error: $\%$) are shown. Best and equivalent methods (under 5$\%$ t-test) are bold.}
\end{center}
\end{figure}

\begin{figure}[htb]
\begin{center}
\scalebox{0.7}[0.7]{
  \begin{tabular}{lc||cccc||cccc}
    \hline
          &  & {\tt 0 vs. 1} & & & & {\tt 0 vs. 2} & & &  \\
    \hline
    True  & Prior & 20 (-) & 40 (-) & 60 (-) & 80 (-) & 20 (-) & 40 (-) & 60 (-) & 80 (-) \\
          & Err & {\bf 0.2} (.004) & 1.0 (.004) & 3.7 (.011) & 8.0 (.020) & {\bf 3.2} (.008) & 8.6 (.016) & 15.0 (.021) & 18.2 (.012) \\
    AltEst1 & Prior & {\bf 19.7} (.007) & 38.6 (.006) &  41.6 (.069) & 53.4 (.065) & {\bf 15.1} (.045) & 34.0 (.037) & 37.0 (.058) & 43.9 (.057) \\
          & Err & {\bf 0.2} (.005) & 1.0 (.005) & 17.7 (.057) & 25.2 (.067) & 5.7 (.034) & {\bf 6.9} (.033) & 24.3 (.060) & 36.5 (.049) \\
    AltEst2 & Prior & {\bf 19.7} (.002) & {\bf 39.6} (.002) & {\bf 59.7} (.003) & {\bf 79.9} (.005) & 12.3 (.017) & 34.5 (.047) &   {\bf 55.4} (.020) & {\bf 79.4} (.017) \\
          & Err & {\bf 0.2} (.001) & {\bf 0.2} (.001) & {\bf 0.4} (.002) & {\bf 1.1} (.006) & 7.1 (.016) & 15.6 (.047) & {\bf 7.7} (.016) & {\bf 9.6} (.015) \\
    KM1   & Prior & 1.3 (.000) & 35.8 (.086) & 60.8 (.013) & 66.5 (.148) & 25.6 (.012) & 48.5 (.015) & {\bf 54.4} (.282) & 21.9 (.063) \\
          & Err & 80.0 (.000) & 4.7 (.072) & 4.9 (.018) & 21.1 (.074) & 4.9 (.010) & 11.6 (.073) & 36.8 (.126) & 60.4 (.066) \\
    KM2   & Prior & 18.0 (.010) & 35.7 (.026) & 51.8 (.017) & 37.6 (.160) & 21.6 (.011) & {\bf 37.1} (.013) & 34.5 (.177) & 1.3 (.000) \\
          & Err & 1.3 (.016) & 2.6 (.034) & 7.9 (.028) & 40.8 (.160) & 3.7 (.009) & 7.0 (.015) & 29.7 (.164) & 80.0 (.000) \\
    \hline
          &  & {\tt 0 vs. 3} & & & &{\tt 0 vs. 4} & & &  \\
    \hline
    True  & Prior & 20 (-) & 40 (-) & 60 (-) & 80 (-) & 20 (-) & 40 (-) & 60 (-) & 80 (-) \\
          & Err & {\bf 1.9} (.005) & 5.9 (.015) & 13.0 (.017) & 16.4 (.013) & {\bf 1.0} (.003) & 3.7 (.009) & 9.7 (.014) & 15.5 (.014) \\
    AltEst1   & Prior & {\bf 17.7} (.024) & 35.5 (.010) & 39.5 (.065) & 43.2 (.049) & 19.3 (.007) & 36.9 (.009) & 38.9 (.065) & 45.8 (.048) \\
          & Err & 2.6 (.018) & {\bf 4.3} (.011) & 20.4 (.070) & 36.2 (.046) & 1.1 (.004) & 2.9 (.008) & 19.6 (.063) & 33.6 (.050) \\
    AltEst2 & Prior & 13.6 (.018) & 31.8 (.044) & {\bf 57.9} (.014) & {\bf 81.7} (.012) & 15.7 (.024) & {\bf 38.2} (.007) &   {\bf 59.6} (.007) & {\bf 80.8} (.008) \\
          & Err & 5.3 (.015) & 7.1 (.040) & {\bf 4.9} (.011) & {\bf 7.3} (.011) & 3.2 (.018) & {\bf 1.5} (.005) & {\bf 2.4} (.006) & {\bf 4.3} (.008) \\
    KM1   & Prior & 25.5 (.010) & 48.5 (.014) & 71.9 (.019) & 3.0 (.057) & 23.4 (.008) & 45.7 (.010) &  68.2 (.015) & 4.2 (.075) \\
          & Err & 3.1 (.008) & 7.7 (.074) & 33.0 (.033) & 78.4 (.056) & 1.2 (.005) & 3.1 (.008) & 23.3 (.027) & 77.0 (.078) \\
    KM2   & Prior & 21.2 (.009) & {\bf 36.5} (.016) & 44.4(.018) & 1.3 (.000) & {\bf 20.2} (.009) & 35.8 (.011) & 44.9 (.022) & 1.3 (.000) \\
          & Err & 2.1 (.005) & 4.7 (.012) & 18.1 (.049) & 80.0 (.000) & 1.2 (.004) & 2.8 (.008) & 13.2 (.055) & 80.0 (.000) \\
    \hline
          &  & {\tt 0 vs. 5} & & & &{\tt 0 vs. 6} & & &  \\
    \hline
    True  & Prior & 20 (-) & 40 (-) & 60 (-) & 80 (-) & 20 (-) & 40 (-) & 60 (-) & 80 (-) \\
          & Err & 4.6 (.005) & 8.9 (.015) & 16.1 (.020) & 18.1 (.011) & {\bf 4.2} (.006) & 8.9 (.014) & 11.9 (.015) & 16.0 (.015) \\
    AltEst1   & Prior & 17.4 (.037) & 34.1 (.047) & 36.3 (.047) & 44.2 (.058) & 14.0 (.052) & 35.1 (.035) & 36.6 (.038) & 49.0 (.052) \\
          & Err & 4.7 (.023) & 8.2 (.039) & 23.8 (.046) & 35.4 (.048) & 7.1 (.032) & 8.3 (.030) & 24.1 (.038) & 31.2 (.050) \\
    AltEst2 & Prior & 12.2 (.032) & 24.8 (.047) & 48.8 (.093) & {\bf 78.8} (.020) & 15.1 (.015) & 34.7 (.029) &   {\bf 59.8} (.012) & {\bf 82.9} (.013) \\
          & Err & 7.2 (.024) & 14.7 (.041) & {\bf 12.9} (.079) & {\bf 12.3} (.015) & 5.6 (.012) & 8.0 (.026) & {\bf 7.3} (.011) & {\bf 8.3} (.011) \\
    KM1   & Prior & 27.8 (.013) & 50.8 (.016) & {\bf 57.9} (.280) & 1.6 (.022) & 23.7 (.014) & 45.6 (.016) & 68.5 (.021) & 11.9 (.129) \\
          & Err & 4.4 (.008) & 23.6 (.111) & 39.6 (.086) & 79.7 (.027) & 4.6 (.006) & 8.7 (.012) & 23.8 (.031) & 69.0 (.138) \\
    KM2   & Prior & {\bf 21.6} (.011) & {\bf 36.6} (.014) & 32.7 (.177) & 1.3 (.000) & {\bf 20.9} (.012) & {\bf 37.5} (.016) & 47.7 (.033) & 3.7 (.060) \\
          & Err & {\bf 3.8} (.005) & {\bf 7.3} (.012) & 30.6 (.168) & 80.0 (.000) & 4.4 (.006) & {\bf 7.8} (.016) & 16.0 (.041) & 78.0 (.052) \\
    \hline
          &  & {\tt 0 vs. 7} & & & &{\tt 0 vs. 8} & & &  \\
    \hline
    True  & Prior & 20 (-) & 40 (-) & 60 (-) & 80 (-) & 20 (-) & 40 (-) & 60 (-) & 80 (-) \\
          & Err & {\bf 2.2} (.007) & 6.5 (.015) & 10.2 (.020) & 15.7 (.018) & {\bf 3.5} (.006) & 8.2 (.012) & 17.0 (.020) & 18.1 (.009) \\
    AltEst1   & Prior & 16.1 (.042) & 36.1 (.021) & 35.3 (.052) & 48.0 (.057) & {\bf 17.6} (.031) & 35.2 (.011) & 38.1 (.060) & 42.3 (.048) \\
          & Err & 4.5 (.033) & 4.9 (.023) & 24.6 (.050) & 31.2 (.056) & 4.3 (.019) & {\bf 6.6} (.010) & 22.1 (.056) & 36.2 (.043) \\
    AltEst2 & Prior & 14.4 (.019) & 36.3 (.010) & {\bf 58.2} (.011) & {\bf 80.3} (.012) & 12.2 (.019) & 25.1 (.035) &   {\bf 54.7} (.034) & {\bf 78.7} (.017) \\
          & Err & 4.7 (.017) & {\bf 3.7} (.008) & {\bf 3.8} (.011) & {\bf 5.9} (.011) & 7.1 (.015) & 14.0 (.032) & {\bf 9.0} (.026) & {\bf 11.8} (.014) \\
    KM1   & Prior & {\bf 20.7} (.011) & {\bf 42.6} (.016) & 65.9 (.022) & 16.5 (.128) & 25.4 (.010) & 48.6 (.014) & 68.0 (.162) & 4.4 (.076) \\
          & Err & 2.6 (.009) & 6.2 (.015) & 19.2 (.031) & 63.6 (.144) & 4.4 (.008) & 11.7 (.091) & 36.7 (.060) & 76.8 (.081) \\
    KM2   & Prior & 18.5 (.011) & 35.3 (.016) & 47.8 (.029) & 8.7 (.090) & 21.1 (.010) & {\bf 36.6} (.015) & 40.9 (.101) & 1.3 (.000) \\
          & Err & 3.0 (.014) & 5.2 (.021) & 14.8 (.052) & 73.4 (.083) & 3.7 (.006) & 7.0 (.012) & 22.9 (.101) &80.0 (.000) \\
  \end{tabular}
  }
\captionsetup{labelformat=empty}
\caption{Table 3: {\tt digits} 0 vs. 1, 2, 3, 4, 5, 6, 7, 8 with dimension $200$ via PCA: The estimated class-prior (Prior: $\%$) and the error rate of classification in test data (Error: $\%$) are shown. Best and equivalent methods (under 5$\%$ t-test) are bold.}
\end{center}
\end{figure}
\begin{figure}
\begin{center}
\scalebox{0.7}[0.7]{
  \begin{tabular}{lc||cccc||cccc}
    \hline
          &  & {\tt mushroom} & & & &{\tt usps} & & &  \\
    \hline
    True  & Prior & 20 (-) & 40 (-) & 60 (-) & 80 (-) & 20 (-) & 40 (-) & 60 (-) & 80 (-) \\
          & Err & 1.0 (.006) & 3.6 (.015) & 6.3 (.021) & 18.9 (.018) & {\bf 7.7} (.009) & {\bf 13.0} (.013) & {\bf 18.1} (.020) & 21.2 (.009) \\
    AltEst1 & Prior & {\bf 18.8} (.006) & 35.1 (.016) & 43.9 (.109) & 51.8 (.160) & {\bf 19.0} (.020) & {\bf 29.6} (.022) & 43.1 (.031) & 56.9 (.036) \\
          & Err & {\bf 0.9} (.005) & 4.3 (.016) & 16.5 (.085) & 29.2 (.127) & 8.0 (.012) & 17.2 (.024) & 23.7 (.027) & 29.4 (.031) \\
    AltEst2 & Prior & 17.0 (.026) & {\bf 37.2} (.013) & {\bf 58.2} (.011) & {\bf 80.1} (.014) & 13.8 (.022) & 27.5 (.031) &  {\bf 46.0} (.090) & {\bf 74.0} (.030) \\
          & Err & 2.1 (.024) & {\bf 1.9} (.012) & {\bf 2.4} (.011) & {\bf 5.0} (.014) & 10.8 (.014) & 18.3 (.024) & 21.6 (.068) & {\bf 18.1} (.023) \\
    KM1   & Prior & 16.9 (.012) & 36.5 (.035) & 43.8 (.050) & 53.2 (.204) & 1.3 (.000) & 17.2 (.050) & 34.7(.027) & 69.9 (.067) \\
          & Err & 1.5 (.020) & 6.7 (.049) & 17.9 (.061) & 21.0 (.158) & 19.5 (.004) & 26.0 (.050) & 29.1 (.029) & 23.6 (.029) \\
    KM2   & Prior & 14.4 (.013) & 28.3 (.027) & 25.6 (.040) & 23.1 (.010) & 70.5 (.011) & 70.0 (.070) & 42.9 (.099) & 22.4 (.009) \\
          & Err & 3.7 (.019) & 11.0 (.046) & 35.2 (.060) & 59.8 (.074) & 80.0 (.000) & 55.8 (.094) & 25.1 (.059) & 71.8 (.020) \\
    \hline
          &  & {\tt waveform} & & & &{\tt spambase} & & &  \\
    \hline
    True  & Prior & 20 (-) & 40 (-) & 60 (-) & 80 (-) & 20 (-) & 40 (-) & 60 (-) & 80 (-) \\
          & Err & {\bf 12.2} (.019) & {\bf 13.8} (.019) & 13.6 (.024) & 12.8 (.024) & {\bf 13.2} (.107) & 56.5 (.063) & {\bf 39.5} (.041) & 30.5 (.081) \\
    AltEst1   & Prior & 41.6 (.035) & 52.2 (.029) & {\bf 61.0} (.109) & 76.2 (.017) & 40.1 (.303) & {\bf 45.3} (.228) & {\bf 56.3} (.156) & 60.1 (.143) \\
          & Err & 22.5 (.048) & 15.4 (.024) & 13.9 (.023) & 12.4 (.025) & 31.8 (.261) & {\bf 33.8} (.145) & 40.6 (.045) & {\bf 25.5} (.076) \\
    AltEst2 & Prior & 41.4 (.035) & 53.5 (.029) & 61.5 (.023) & 76.4 (.047) & 51.7 (.332) & 51.1 (.272) &  {\bf 57.1} (.183) & 59.4 (.148) \\
          & Err & 23.3 (.049) & 15.5 (.025) & 14.2 (.023) & 13.2 (.036) & 45.0 (.257) & 41.4 (.124) & 43.0 (.051) & {\bf 26.1} (.094) \\
    KM1   & Prior & {\bf 30.8} (.010) & {\bf 49.1} (.012) & 67.2 (.013) & 85.7 (.204) & {\bf 33.4} (.041) & 52.7 (.041) & 71.2 (.046) & {\bf 86.5} (.022) \\
          & Err & 15.6 (.022) & 15.0 (.021) & {\bf 12.8} (.021) & {\bf 9.8} (.023) & 61.6 (.152) & 48.0 (.040) & 40.4 (.038) & 30.5 (.089) \\
    KM2   & Prior & 34.4 (.009) & 50.5 (.011) & 66.0 (.012) & {\bf 78.7} (.014) & 42.2 (.035) & 54.5 (.038) & 63.3 (.054) & 61.4 (.107) \\
          & Err & 17.9 (.026) & 14.8 (.021) & 13.3 (.022) & 13.1 (.024) & 55.6 (.057) & 48.5 (.033) & 40.0 (.037) & {\bf 25.8} (.094) \\
  \end{tabular}
  }
  \captionsetup{labelformat=empty}
  \caption{Table 4: {\tt mushrooms}, {\tt usps}, {\tt waveform} and {\tt spambase} ({\tt mushrooms} and {\tt usps} with dimension $100$ via PCA): The estimated class-prior (Prior: $\%$) and the error rate of classification in test data (Error: $\%$) are shown. Best and equivalent methods (under 5$\%$ t-test) are bold.}
  \end{center}
\end{figure}

\clearpage

\appendix

\section{Proof of Lemma 1}
To prove the lemma, the KKT condition is important. Here, we briefly explain the KKT condition. We consider the following  optimization problem with continuous convex functions, $a(x)$ and $b(x)$:
\begin{align*}
  \min_{x} &\ a(x)\\
  s.t. & \ b(x) \leq 0.
\end{align*}
Then we consider its Lagrange function:
\begin{align*}
  \inf_{x} \sup_{\lambda \geq 0}a(x) + \lambda b(x),
\end{align*}
where $\lambda$ is the Lagrangian multiplier. If $b(x) > 0$, $\lambda b(x) \rightarrow +\infty$ as $\lambda \rightarrow +\infty$.  If $b(x) \leq 0$, $\lambda b(x) = 0$ when $\lambda = 0$ as a result of minimization. As a result, we can derive the KKT condition as follows:
\begin{align*}
  &a'(x^{*}) + \lambda b'(x^{*}) = 0,\\
  &b(x^{*}) \leq 0,\\
  &\lambda \geq 0,\\
  &\lambda b(x^{*}) = 0,
\end{align*}
where $x^{*}$ is the optimal value.

We apply the KKT condition for functionals to the proof of Lemma 1.

\begin{proof}
Let us consider maximizing $J$ for $f \in (0,1-\epsilon]$. We introduce Lagrangian variables $\alpha(\bm{x})$ and $\beta(\bm{x})$, which are functions $\mathbb{R}^d \to \mathbb{R}$ with range $(0, +\infty)$, and the Lagrange functional 
\begin{align*}
    &\mathcal{L}^{\pi^{(k)}}(f;\alpha,\beta)\nonumber\\
    & = \int \Big\{-\pi^{(k)}p(\bm{x}|y=+1)(\log f(\bm{x}) - \log (1-f(\bm{x})))\nonumber\\
    & - \log (1-f(\bm{x}))p(\bm{x}) + \alpha(\bm{x})(-1+\epsilon + f(\bm{x}))-\beta(\bm{x})f(\bm{x}) \Big\}\mathrm{d}\bm{x}.
\end{align*}
Then we consider the maximization of $\mathcal{L}$ with respect to $\alpha$ and $\beta$, i.e.,
\begin{align*}
  \inf_{f} \sup_{\alpha, \beta}\mathcal{L}^{\pi^{(k)}}(f;\alpha,\beta).
\end{align*}
Now we obtain 
\begin{align*}
  \inf_{f}\mathcal{L}^{\pi^{(k)}}(f;\alpha,\beta) = 
  \begin{cases}
    R^{\pi^{(k)}}(f) & if\ 0\leq f \leq 1-\epsilon\\
    +\infty & otherwise.
  \end{cases}
\end{align*}
Therefore, 
\begin{align*}
\inf_{f, 0\leq f\leq 1-\epsilon} R^{\pi^{(k)}}(f) = \inf_{f} \sup_{\alpha, \beta}\mathcal{L}^{\pi^{(k)}}(f;\alpha,\beta).
\end{align*}
Next, we consider the dual problem defined as follows:
\begin{align*}
  \sup_{\alpha, \beta}\inf_{f}\mathcal{L}^{\pi^{(k)}}(f;\alpha,\beta).
\end{align*}
Because of the convexity of $R^{\pi^{(k)}}(f)$ for $f$, the following equality can be obtained \cite{rockafellar2015convex}.
\begin{align*}
  \inf_{f} \sup_{\alpha, \beta}\mathcal{L}^{\pi^{(k)}}(f;\alpha,\beta) = \sup_{\alpha, \beta}\inf_{f}\mathcal{L}^{\pi^{(k)}}(f;\alpha,\beta).
\end{align*}
Hence, we discuss the solution of the dual problem.

Given $\alpha$ and $\beta$, we apply the Euler-Lagrange equation for calculating $\inf_{f} \mathcal{L}(f;\alpha,\beta)$ \cite{gelfand2000calculus}. The maximizer $f^{(k+1)*}(\bm{x})$ satisfies the following equation,
\begin{align*}
  &\pi^{(k)}p(\bm{x}|y=+1)\left(\frac{1}{f^{(k+1)*}(\bm{x})}+\frac{1}{1-f^{(k+1)*}(\bm{x})} \right) - \frac{p(\bm{x})}{1-f^{(k+1)*}(\bm{x})} - \alpha(\bm{x}) + \beta(\bm{x}) = 0.
\end{align*}
Therefore, 
\begin{align}
\label{equation}
  &(\beta(\bm{x})-\alpha(\bm{x}))(f^{(k+1)*}(\bm{x}))^2 + (p(\bm{x})+\alpha(\bm{x})-\beta(\bm{x}))f^{(k+1)*}(\bm{x}) - \pi^{(k)}p(\bm{x}|y=+1) = 0.
\end{align}
As a result, we can derive the following KKT condition:
\begin{align*}
  &(\beta^{(k+1)*}(\bm{x})-\alpha^{(k+1)*}(\bm{x}))(f^{(k+1)*}(\bm{x}))^2 + (p(\bm{x})+\alpha^{(k+1)*}(\bm{x}) -\beta^{(k+1)*}(\bm{x}))f^{(k+1)*}(\bm{x}) - \pi^{(k)}p(\bm{x}|y=+1) = 0,\\
  &f^{(k+1)*}(\bm{x}) \leq 1-\epsilon,\\
  &f^{(k+1)*}(\bm{x}) \geq 0,\\
  &\alpha^{(k+1)*}(\bm{x})) \geq 0,\\
  &\beta^{(k+1)*}(\bm{x}) \geq 0,\\
  &\alpha^{(k+1)*}(\bm{x})(f^{(k+1)*}(\bm{x}) - 1-\epsilon) = 0,\\
  &\beta^{(k+1)*}(\bm{x}) f^{(k+1)*}(\bm{x}) = 0,
\end{align*}
where $f^{(k+1)*}(\bm{x})$, $\alpha^{(k+1)*}(\bm{x})$ and $\beta^{(k+1)*}(\bm{x})$ are the optimal functions, and the following results:
\begin{enumerate}
  \item If $\alpha^{(k+1)*}(\bm{x}) = \beta^{(k+1)*}(\bm{x}) = 0$, by (\ref{equation}), $f^{(k+1)*}(\bm{x})=\frac{\pi^{(k)}p(\bm{x}|y=+1)}{p(\bm{x})}$. This solution satisfies the constraint only when $\bm{x}$ is in the domain $D^{(k)}$, where $D^{(k)} = \{x|\pi^{(k)}p(\bm{x}|y=1) \leq (1-\epsilon)p(\bm{x}) \}$. 
  \item If $\alpha^{(k+1)*}(\bm{x}) > 0$ and $\beta^{(k+1)*}(\bm{x}) = 0$, by (\ref{equation}), $f^{(k+1)*}(\bm{x})=1-\epsilon$. It is because $\inf_{f}\mathcal{L}(f;\alpha,\beta) \rightarrow +\infty$ as $\alpha^{(k+1)*}(\bm{x}) \rightarrow +\infty$ if $f^{(k+1)*}(\bm{x})<1-\epsilon$. $\alpha^{(k+1)*}(\bm{x}) > 0$ holds only when $\bm{x}$ is in the domain $\mathbb{R}\backslash D^{(k)}$, where $\mathbb{R}\backslash D^{(k)} = \{x|\pi^{(k)}p(\bm{x}|y=1) < (1-\epsilon)p(\bm{x}) \}$, because 
  \begin{align*}
  &-\alpha^{(k+1)*}(\bm{x})(1-\epsilon)^2+ (p(\bm{x})+\alpha^{(k+1)*}(\bm{x}))(1-\epsilon) - \pi^{(k)}p(\bm{x}|y=+1)=0\\
  &\Leftrightarrow\ -\alpha^{(k+1)*}(\bm{x})(1-\epsilon)^2+ \alpha^{(k+1)*}(\bm{x})(1-\epsilon)+p(\bm{x})(1-\epsilon) - \pi^{(k)}p(\bm{x}|y=+1) = 0\\
  &\Leftrightarrow\ \alpha^{(k+1)*}(\bm{x}) = \frac{p(\bm{x})(1-\epsilon) - \pi^{(k)}p(\bm{x}|y=+1)}{(1-\epsilon)^2+ (1-\epsilon)}.
  \end{align*}
  \item If $\alpha^{(k+1)*}(\bm{x}) = 0$ and $\beta^{(k+1)*}(\bm{x}) > 0$, by (\ref{equation}), $f^{(k+1)*}(\bm{x})=0$. It is because $\inf_{f}\mathcal{L}(f;\alpha,\beta) \rightarrow +\infty$ as $\beta(\bm{x}) \rightarrow +\infty$ if $f^{(k+1)*}(\bm{x}) > 0$. 
  This solution satisfies (\ref{equation}), only when $\pi^{(k)}p(\bm{x}|y=+1)=0$ because 
  \begin{align*}
  &\beta^{(k+1)*}(\bm{x})(0)^2 + (p(\bm{x})\nonumber\\
  &-\beta^{(k+1)*}(\bm{x}))0 - \pi^{(k)}p(\bm{x}|y=+1) = 0.
\end{align*}
\end{enumerate}
Otherwise, there is no feasible solution. As a result, the solution for the optimization problem $f^{(k+1)*}(\bm{x}) = \arg\min_{f(\bm{x})\in(0,1-\epsilon]} R^{\pi^{(k)}}(f)$ is
\begin{align*}
    f^{(k+1)*}(\bm{x}) =  \begin{cases}
      \frac{\pi^{(k)}p(\bm{x}|y=+1)}{p(\bm{x})} &(\bm{x}\in D^{(k)}),\\
       1-\epsilon &(\bm{x}\notin D^{(k)}),
    \end{cases}
\end{align*}
where $D^{(k)} = \{\bm{x}|\pi^{(k)}p(\bm{x}|y=1) \leq (1-\epsilon)p(\bm{x}) \}$.
\end{proof}

\section{Proof of Lemma 2}
\begin{proof}
By definition, 
\begin{align*}
&\pi^{(k)} =  \int \frac{\pi^{(k)}p(\bm{x}|y=+1)}{p(\bm{x})}p(\bm{x})\mathrm{d}\bm{x}, \nonumber\\
& \pi^{(k+1)} =  \int f^{*}(\bm{x})p(\bm{x})\mathrm{d}\bm{x}.
\end{align*}
    Then, we have
    \begin{align*}
        &\pi^{(k)} - \pi^{(k+1)}\nonumber\\
        &\quad = \int \left(\frac{\pi^{(k)}p(\bm{x}|y=+1)}{p(\bm{x})} - f^*(\bm{x})\right)p(\bm{x}) \mathrm{d}\bm{x}\nonumber\\
        &\quad = \int_{D} \left(\frac{\pi^{(k)}p(\bm{x}|y=+1)}{p(\bm{x})} - f^*(\bm{x})\right)p(\bm{x}) \mathrm{d}\bm{x} + \int_{\mathbb{R}\backslash D} \left(\frac{\pi^{(k)}p(\bm{x}|y=+1)}{p(\bm{x})} - f^*(\bm{x})\right)p(\bm{x}) \mathrm{d}\bm{x}\nonumber\\
        &\quad = \int_{\mathbb{R}\backslash D} \left(\frac{\pi^{(k)}p(\bm{x}|y=+1)}{p(\bm{x})} - (1-\epsilon)\right)p(\bm{x}) \mathrm{d}\bm{x}\nonumber\\
        &\quad = \int \left({\pi^{(k)}p(\bm{x}|y=+1)} - (1-\epsilon)p(\bm{x})\right)_+ \mathrm{d}\bm{x}.
    \end{align*}
    $\int_{\mathbb{R}\backslash D} \left(\frac{\pi^{(k)}p(\bm{x}|y=+1)}{p(\bm{x})} - (1-\epsilon)\right)p(\bm{x}) \mathrm{d}\bm{x}$ is strictly positive because it takes the integral on the domain of $\bm{x}$, which is $\{\bm{x}|\pi^{(k)}p(\bm{x}|y=1) - (1-\epsilon)p(\bm{x}) > 0 \}$.
\end{proof}

\section{Proof of Theorem 1}
\begin{proof}
Before showing the proof of the theorem, we prove the following lemma.
\begin{lmm}
For $a \geq b>0$ and $n \in \{1,2,\dots\}$, the following inequality holds.
\begin{align*}
    \frac{1}{a^n} - \frac{b}{a^{n+1}} \leq  \frac{1}{b^n} - \frac{1}{a^n}.
\end{align*}
\end{lmm}
\begin{proof}
\begin{align*}
    a \geq b &\Leftrightarrow a^n \geq b^n\\
    &\Leftrightarrow a^n(a-b) \geq b^n(a-b)\\
    &\Rightarrow a^n\left(a-b\left(\frac{b}{a}\right)^{n-1}\right) \geq b^n(a-b)\\
    &\Leftrightarrow \frac{a}{b^n} - \frac1{a^{n-1}} \geq \frac{1}{a^{n-1}} - \frac{b}{a^n}\\
    &\Leftrightarrow \frac{1}{b^n} - \frac1{a^n} \geq \frac{1}{a^n} - \frac{b}{a^{n+1}}.
\end{align*}
\end{proof}
Then, we prove $\pi^{(k)} \geq \pi_{\max}$. This can be proved by mathematical induction. Let $f^{(k)}(\bm{x})$ be a classifier learned at the $k$-th round, namely,
\begin{align*}
    f^{(k)}(\bm{x}) = \min\left(\frac{\pi^{(k)} p(\bm{x}|y=+1)}{p(\bm{x})}, 1-\epsilon \right).
\end{align*}
Then, if $\pi^{(k)} \geq \pi_{\max}$, we have
\begin{align*}
    \frac{\pi^{(k)} p(\bm{x}|y=+1)}{p(\bm{x})} \geq \frac{\pi_{\max} p(\bm{x}|y=+1)}{p(\bm{x})}.
\end{align*}
Since $1-\epsilon \geq \frac{\pi_{\max} p(\bm{x}|y=+1)}{p(\bm{x})}$ holds by the definition of $\pi_{\max}$, we have
\begin{align*}
    \pi^{(k+1)} &= \int f(\bm{x})p(\bm{x}) \mathrm{d}\bm{x}\\
    &= \int \min\left(\frac{\pi^{(k)} p(\bm{x}|y=+1)}{p(\bm{x})}, 1-\epsilon \right) p(\bm{x}) \mathrm{d}\bm{x}\\
    &\geq \int \frac{\pi_{\text{max}} p(\bm{x}|y=+1)}{p(\bm{x})} p(\bm{x}) \mathrm{d}\bm{x}\\
    &= \pi_{\max}.
\end{align*}
Therefore, $\pi^{(k)} \geq \pi_{\max} \Rightarrow \pi^{(k+1)} \geq \pi_{\max}$. Considering that $\pi^{(0)} \geq \pi_{\max}$, we have $\pi^{(k)} \geq \pi_{\max}$ for all $k>0$.

Based on this fact, we can prove the convergence. By the definition of $\pi_{\max}$, there exists $\bm{x}^* \in \mathbb{R}^d$ such that
\begin{align*}
    \frac{\pi_{\max} p(\bm{x}^*|y=+1)}{p(\bm{x}^*)} = 1-\epsilon.
\end{align*}
Hence, let $g^{(k)}(\bm{x})$ be $g^{(k)}(\bm{x}) = {\pi^{(k)}p(\bm{x}|y=+1)} - (1-\epsilon)p(\bm{x})$, and we have
\begin{align*}
    &g^{(k)}(\bm{x}^*)= {\pi^{(k)}p(\bm{x}^*|y=+1)} - (1-\epsilon)p(\bm{x}^*)= (\pi^{(k)}-\pi_{\max})p(\bm{x}^*|y=+1).
\end{align*}
Since the assumption of Lipchitz continuity of probability density functions, $g^{(k)}(\bm{x})$ is $(\pi^{(k)}L_1 + (1-\epsilon)L_2)$-Lipchitz continuous, which implies
\begin{align*}
    g^{(k)}(\bm{x}) &\geq (\pi^{(k)}-\pi_{\max})p(\bm{x}^*|y=+1) - (\pi^{(k)}L_1 + (1-\epsilon)L_2)\|\bm{x}-\bm{x}^*\|_1\\
    &\geq (\pi^{(k)}-\pi_{\max})p(\bm{x}^*|y=+1) - (L_1 + L_2)\|\bm{x}-\bm{x}^*\|_1.\\
\end{align*}
Let $y_k$ be $y_k =\pi^{(k)}-\pi_{\max}$ and $p^*$ be $p^*=p(\bm{x}^*|y=+1)$. Then we have
\begin{align*}
    y_k - y_{k+1} &= \int (g^{(k)}(\bm{x}))_+ \mathrm{d}\bm{x} \quad(\because\text{Lemma 2})\\
    &\geq \int \left((\pi^{(k)}-\pi_{\max})p^* - (L_1 +L_2)\|x-x^*\|_1\right)_+ \mathrm{d}\bm{x}\\
    &\geq \int_{\|\bm{x}-\bm{x}^*\|_1 \leq \frac{y_k p^*}{2(L_1 + L_2)}} \frac{y_kp^*}2 \mathrm{d}\bm{x}\\
    &= \frac{y_k^{d+1}(p^*)^{d+1}}{2^{d+2}(L_1 + L_2)^d} > 0,
\end{align*}
where $d$ is the dimension of $\bm{x}$. Considering that $\pi^{(k)} > \pi_{\max}$, we have $y_k > y_{k+1} > 0$. Therefore, using Lemma 3, we have
\begin{align*}
    &y_k - y_{k+1} \geq \frac{y_k^{d+1}(p^*)^{d+1}}{2^{d+2}(L_1 + L_2)^d}\\
    \Leftrightarrow &\frac{1}{(y_k)^d} - \frac{y_{t+1}}{(y_k)^{d+1}} \geq \frac{(p^*)^{d+1}}{2^{d+2}(L_1 + L_2)^d}\\
    \Rightarrow &\frac{1}{(y_{k+1})^d} - \frac{1}{(y_k)^{d}} \geq \frac{(p^*)^{d+1}}{2^{d+2}(L_1 + L_2)^d}.
\end{align*}
Taking the sum of both sides for $t=1,\dots,T-1$ yields
\begin{align*}
    &\frac{1}{(y_T)^d} - \frac{1}{(y_1)^{d}} \geq \frac{(T-1)(p^*)^{d+1}}{2^{d+2}(L_1 + L_2)^d}\nonumber\\
    \Leftrightarrow & y_T \leq \frac{1}{\left(\frac{(T-1)(p^*)^{d+1}}{2^{d+2}(L_1 + L_2)^d} + \frac{1}{(y_1)^d}\right)^{\frac1d}}.
\end{align*}
Combined (\ref{equation}) with $y_T \geq 0$, we can conclude that $y_T \to 0$ as $T\to +\infty$, which implies $\pi^{(k)} \to \pi_{\max}$.
\end{proof}

\end{document}